%% file: 1_main.tex
\documentclass[accepted]{uai2023} 

\usepackage[american]{babel}

\usepackage{natbib} 
\usepackage{mathtools} 
\usepackage{booktabs} 
\usepackage{xcolor}
\usepackage{bm}
\usepackage{tikz}
\usetikzlibrary{arrows,shapes,automata,petri,positioning,calc}
\usepackage{cancel}
\usepackage{amsmath}
\usepackage{amsfonts}
\usepackage{amsthm}
\usepackage{wrapfig}
\usepackage{tabularx}
\usepackage{hyperref}
\usepackage[ruled,vlined,linesnumbered,lined,commentsnumbered,noend]{algorithm2e}
\usepackage{xspace}
\usepackage{float}

\newcommand\commentfont[1]{\ttfamily\textcolor{blue}{#1}}
\SetCommentSty{commentfont}

\newcommand{\pluseq}{\mathrel{+}=}

\newtheorem{theorem}{Thm.}

\newenvironment{proofsketch}{\paragraph{Proof Sketch:}}{\hfill$\square$}

\title{Equitable Restless Multi-Armed Bandits:\\ A General Framework Inspired By Digital Health}

\author[1]{\href{mailto:<jkillian@g.harvard.edu>}{Jackson A.~Killian}{}}
\author[2]{Manish Jain}
\author[3]{Yugang Jia}
\author[3]{Jonathan Amar}
\author[3]{Erich Huang}
\author[1,2]{Milind Tambe}
\affil[1]{%
    Computer Science\\
    Harvard University
}
\affil[2]{%
    Google Research
}
\affil[3]{%
    Verily Life Sciences
}
  
\begin{document}
\maketitle

\begin{abstract}
Restless multi-armed bandits (RMABs) are a popular framework for algorithmic decision making in sequential settings with limited resources. RMABs are increasingly being used for sensitive decisions such as in public health, treatment scheduling, anti-poaching, and --- the motivation for this work --- digital health. For such high stakes settings, decisions must both improve outcomes and prevent disparities between groups (e.g., ensure health equity). We study equitable objectives for RMABs (ERMABs) for the first time. We consider two equity-aligned objectives from the fairness literature, minimax reward and max Nash welfare. We develop efficient algorithms for solving each --- a water filling algorithm for the former, and a greedy algorithm with theoretically motivated nuance to balance disparate group sizes for the latter. Finally, we demonstrate across three simulation domains, including a new digital health model, that our approaches can be multiple times more equitable than the current state of the art without drastic sacrifices to utility. Our findings underscore our work's urgency as RMABs permeate into systems that impact human and wildlife outcomes. Code is available at \url{https://github.com/google-research/socialgood/tree/equitable-rmab}
\end{abstract}

\section{Introduction}
\input{2_intro.tex}

\section{Related Work}
We study RMABs, the \textit{restless} generalization of stochastic multi-armed bandits (MABs) in which arms follow Markov decision processes. RMABs have general application and are widely studied, e.g., sensor/machine maintenance \citep{abbou2019group,villar2016indexability}, wireless network scheduling \citep{modi2019transfer,cohen2014restless}, anti-poaching patrol planning \citep{qian2016restless}, and various public health contexts \citep{mate2022field,deo2013improving,lee2019optimal}. 
Fairness in RMAB has only recently been studied, mainly via \textit{equality}, i.e., ensuring all arms have a lower bound probability of receiving an intervention \citep{herlihy2021planning,li2022towards}. Alternatively, \citet{mate2021risk} view fairness as allowing planners to shape rewards to encode their relative risk-averseness. Our work is the first to consider \textit{equity}-focused objectives, viewing fairness through the lens of \textit{equal outcomes}.

Our work also relates to fairness in stochastic MABs. Similarly, much work has focused on equality; \citet{jeunen2021top} ensure each arm receives a minimum threshold of ``exposure''
and \citet{patil2021achieving} give fairness guarantees for the minimum pull threshold problem. \citet{liu2017calibrated} assign arm pulls with probability proportional to each arm's expected reward; note that this ignores whether an arm \textit{needs} a pull to reach higher reward, working against our equity objective. More similar to our setting, \citet{barman2022fairness} take a welfarist approach that ensures even distribution of rewards over \textit{time}, ignoring fairness over arms; we study the more general setting in which fairness must also be ensured over arms. Most related is \citet{ron2021corporate};
the planner specifies a fairness function over all arm rewards, then learns a utilitarian policy which trades off between penalties for violating fairness, versus reward maximization. 
Ultimately their approach determines some minimum set of pulls for each arm. Conversely, in our restless setting, arms have state which evolve with or without arm pulls, and the planner must equitably respond to real-time state changes, a far more complex planning challenge.


\section{Preliminaries}
Restless bandits have $n \in 1,...,N$ arms, discrete per-arm state space $\mathcal{S}_n$, per-arm action space $\mathcal{A}_n = \{0, 1\}$, per-arm transition functions $P_n$ defining the probability of arm $n$ transitioning from state $s$ to state $s^\prime$ given action $a$, per-arm reward function $R_n(s)$ defining the reward for an arm in state $s$, time horizon $H$, and action budget $b$. For ease of exposition, $\mathcal{S}_n$, $\mathcal{A}_n$ and $R_n(s)$ are the same for all arms, so we drop the subscript $n$, but our methods apply to the general setting. Let $\bm{s}^t$ be the $N$-length vector of arm states at time $t$, indexed as $s_n^t$, and let $\bm{a}^t$ be an $N$-length one-hot encoding of the arms that receive actions at time $t$, indexed as $a_n^t$. The planner computes $\pi$, a map from $\bm{s}^t$ to $\bm{a}^t$, subject to per-round budget constraints, $|\bm{a}^t|_1 \le b \hspace{1mm} \forall t \in 1,...,H$.

The objective of a traditional utility-maximizing RMAB is to find a policy $\pi$ that maximizes
\begin{equation}\label{eq:rmab_objective}
     \underset{\bm{s}^{t+1}\sim P(\bm{s}^{t}, \pi(\bm{s}^{t}), \cdot)}{\mathbb{E}}\sum_{t=0}^{H-1} \sum_{n=1}^{N}R(s_n^t)
\end{equation}
given some $\bm{s}^0$. $\pi$ is found by computing $V^0(\bm{s}^0, b)$, where:
\begin{equation}\label{eq:rmab_value_function}
    V^t(\bm{s}^t, b) = \max_{\bm{a}^t}
    \left\{\sum_{n=1}^{N}R(s_n^t) + \mathbb{E}[V^{t+1}(\bm{s}^{t+1}) | \bm{s}^t, \bm{a}^t] \right\}
\end{equation}
\begin{equation}\label{eq:rmab_constraint}
    \text{s.t.}\hspace{1mm}\sum_{n=1}^{N}a_n^k \le b \hspace{2mm} \forall k \in t,...,H
\end{equation}
and $V^H(\cdot)=0$. However, solving this is PSPACE-Hard \citep{papadimitriou1999complexity} due to the exponential state and action spaces resulting from budget coupling. Instead, the Whittle index policy is commonly used~\citep{whittle1988restless}, which derives from the Lagrangian relaxation of Eq.~\ref{eq:rmab_constraint}:
\begin{equation}\label{eq:lagrange_objective}
    L^t(\bm{s}^t, b) = \min_{\bm{\lambda},V^{k\in[t,...,H]}_{n\in[1,...,N]}} \sum_{n=1}^{N}V_{n}^{t}(s_n^t, \lambda^t) + b\sum_{k=t}^{H}\lambda^k
\end{equation}
\begin{align}\label{eq:lagrange_constraints}
    \text{s.t.}\hspace{1mm}
    V_n^k(s_n^k, \lambda) \ge R&(s_n^k) - a_{nj}^{k}\lambda^k + \nonumber\\
    &\sum_{s^\prime \in \mathcal{S}} V_n^{k+1}(s^\prime, \lambda)P(s_n^k, a_{nj}^k, s^\prime) \nonumber \\ 
    \forall k \in t,...,H-1, \hspace{1mm} &\forall j \in \{0, 1\}, \hspace{1mm} \forall s_n^k \in \mathcal{S}, \hspace{2mm} \forall n \in 1,...,N \nonumber
\end{align}
\begin{equation}\label{eq:lagrange_basecase}
    V_n^H(s_n, \lambda) = 0 \hspace{2mm} \forall s_n \in \mathcal{S}, \hspace{2mm} \forall n \in 1,...,N \nonumber
\end{equation}
The Whittle index of an arm in state $s$, at time $t$, is $W^t_n(s)$ and is equal to the $\lambda$, such that $\lambda^k = \lambda \ \forall k$, that makes both constraints $j \in \{0,1\}$ tight for that $V_n^t(s,\lambda)$. In other words, $W^t_n(s)$ is the action charge $\lambda$ such that both actions have equal value. Intuitively, $W^t_n(s)$ measures an arm's long-term budget efficiency, giving high values for arms that \textit{need} budget to produce reward but require \textit{little budget} to do so. 

The solution to Eq.~\ref{eq:lagrange_objective} is to compute the $\bm{\lambda}$ that induces the optimal policies of each arm to spend $bH$ budget in expectation. The Whittle index follows a related process of identifying a $\lambda$ which ensures that $b$ arms are acted on in the current round. Both policies perform extremely well in practice and the Whittle index policy is asymptotically optimal under technical conditions \citep{weber1990index}.

However, such policies which prioritize only the "most efficient" arms lead to inequitable solutions. An arm that may benefit from budget but that is seen as slightly less efficient may never receive budget. In a strictly utilitarian world, this may be optimal, but in settings that impact human outcomes such as healthcare, we must consider more complex objectives that are sensitive to the distribution of outcomes. No such tools yet exist in the RMAB literature. 

\section{Equitable Restless Bandits}
\subsection{Equitable Objectives}
Equitable objectives reason prospectively about \textit{outcomes}, prioritizing resource allocations that lead to well-balanced outcomes across groups. However, for RMABs, computing outcomes is itself PSPACE-Hard, setting this apart from existing literature on fair resource allocation which has few tools for optimal control. However, we identify tractable structure in the objectives below, leading to a key advance that can realign RMABs toward social objectives.

We consider equity across pre-defined groups of arms $\mathcal{G}$, indexed by $g$. Let $M: N\rightarrow{}\mathcal{G}$ be a surjective mapping of arms to groups and let $M^{-1}(g)$ be the set of arms in group $g$. Let the value function of a group $g$, given $|M^{-1}(g)|$-sized state vector $\bm{s}_g^t$ at time $t$ with H-length vector of per-round budgets $\bm{b}_g$ be $V_g^t(\bm{s}_g^t, \bm{b}_g)$ and the corresponding Lagrangian relaxation for a group be $L^t_g(\bm{s}_g^t, \bm{b}_g)$.
Our goal is to find policies that are equitable with respect to outcomes across groups, as measured by relative differences in each group's reward at the end of the horizon. The general form of an equitable RMAB is as follows:
\begin{equation}\label{eq:general_equity_objective}
        \max_{\bm{b}_g} f(V^0_1, V^0_2, ..., V^0_{|\mathcal{G}|})
\end{equation}
\begin{equation}\label{eq:general_equity_constraints}
    \sum_{g \in \mathcal{G}}\bm{b}_g^t = B \hspace{2mm} \forall t \in [0,...,H-1]
\end{equation}
Where $f$ encodes the equity function and $B$ is the total per-round budget constraint over all groups. Next we introduce two choices for $f$, discuss their properties and theoretical motivation, and develop algorithms for optimizing them in the context of ERMABs.
\subsubsection{MaxiMin Reward}
\begin{equation}\label{eq:maximin_reward_objective}
        \max_{\bm{b}_g, r^\star} r^\star
\end{equation}
\begin{equation}\label{eq:maximin_reward_constraints}
    \text{s.t. } V^0_g(\bm{s}^0_g, \bm{b}_g) \ge r^\star \hspace{1mm} \forall g \in \mathcal{G}
\end{equation}
\begin{equation}\label{eq:maximin_reward_budget_constraints}
    \sum_{g \in \mathcal{G}}\bm{b}_g^t = B \hspace{2mm} \forall t \in [0,...,H-1]
\end{equation}
MaxiMin reward (MMR) is a robust objective that maximizes the minimum prospective total reward of any group. This egalitarian approach to equity is well studied, especially in the context of the allocation of goods (readers are encouraged to reference \citep{luss2012equitable}). Variants of maximin reward (e.g., lexicographic maximin) which equate to Eq.~\ref{eq:maximin_reward_objective} are desirable for their Pareto optimality, given monotonicity conditions on the utility functions. That is, given a solution to Eq. ~\ref{eq:maximin_reward_objective}, the solution cannot be permuted in a way such that the utility of one group is increased without decreasing the utility of another. Moreover, in some cases, maximin objectives are desirable for their uncompromising approach to ensuring equal outcomes; the groups that are the worst-off are prioritized for resources without exception. However, in other cases it is precisely this lack of flexibility for which maximin has been criticized. For instance, consider a case with two groups with utility functions $V_1(b) = C(b+1)$ and $V_2(b) = \epsilon (b+1)$ where $C$ is an arbitrarily large positive constant and $\epsilon$ is an arbitrarily small positive constant. Here, maximin will find a solution that is \textit{both} arbitrarily worse than the utility maximizing solution and arbitrarily unequal in outcomes. In such a case where inequity is inherently unachievable, it may desirable to recover some utilitarian efficiency. Thus, it is useful to also consider equity objectives which are capable of incorporating more flexibility between equality of outcomes and efficiency, which we describe next.

\subsubsection{Maximum Nash Welfare}
\begin{equation}\label{eq:mnw_objective}
        \max_{\bm{b}_g} \prod_{g \in \mathcal{G}} V^0_g(\bm{s}^0_g, \bm{b}_g)
\end{equation}
\begin{equation}\label{eq:mnw_budget_constraints}
    \sum_{g \in \mathcal{G}}\bm{b}_g^t = B \hspace{2mm} \forall t \in [0,...,H-1]
\end{equation}
Maximum Nash welfare (MNW) optimizes the product of group outcomes. This objective is also well studied in the context of the allocation of goods \citep{moulin2004fair,caragiannis2019unreasonable} and is desirable for its theoretic properties, i.e., Pareto optimality under monotonic value functions, as well as its ability to find solutions that naturally tradeoff between equity and efficiency. It achieves this by assigning diminishing returns for each additional marginal increase in the utility of any value function. This naturally leads to allocations which are relatively balanced, as well as (log-)proportionally scaled according to each group's expected utility increase from additional allocations. We can see the effect of this with an example: 
consider $V_1(b) = 2b+1$ and $V_2(b) = 4(b+1)$ with $B=2$. MNW splits budget evenly giving $(V_1(1), V_2(1))=(3,8)$, whereas MMR puts all budget in the worst-off group giving $(V_1(2), V_2(0))=(5,4)$, and where max utility puts all budget on the most efficient group giving $(V_1(0), V_2(2))=(1,12)$. Observe that the MNW solution ranks second of the three solutions in terms of group utility deviation ($=5$), and the sum of group utilities ($=11$).

\subsubsection{Metrics and Important Considerations}
\label{sec:metrics}
To measure the equity of our policies, we will employ the widely used measure from economics known as the Gini index \citep{gini1909}, which measures the mean average deviation of members of a population --- 0 is perfect equality and 1 is perfect inequality. However, since we are interested in \textit{group} fairness, where groups may be of different sizes, we report the gini index of per-group average outcomes, i.e., Gini$(V_1 / |M^{-1}(1)|, V_2 / |M^{-1}(2)|, ...,)$. 

Additionally, there are three important considerations to be made before we can solve Eq.~\ref{eq:maximin_reward_objective} and Eq.~\ref{eq:mnw_objective}. First, is that computing $V^0_g(\bm{s}^0_g, \bm{b}_g)$ exactly is PSPACE hard \citep{papadimitriou1999complexity}, and thus intractable. Therefore, our approach will be to substitute the tractable upper bound $L_g^0(\bm{s}_g^0,\bm{b}_g)$, then bound the difference. Second, since we are interested in group-size normalized outcomes, we need to incorporate group size into each objective. For MMR, this is as simple as optimizing $L_g^0(\bm{s}_g^0,\bm{b}_g) / |M^{-1}(g)|$. However, it is slightly trickier for MNW. Since MNW is scale-invariant, normalizing by group size does not change allocations. In fact, interestingly, in the presence of unequal sized groups, MNW is heavily biased toward small groups. In section \ref{section:algorithm_mnw}, we discuss why and how to adjust for this in the group-average case we consider. 
Finally, to develop algorithms for solving each objective that retain their desirable properties, we must understand the structure of the utility functions. We show in the next section that both $V$ and $L$ have a desirable monotone increasing form.

\subsection{Problem Structure}

\begin{theorem}\label{thm:V_monotonic}
$V_g^0(\cdot,b)$ increases monotonically in $b$.
\end{theorem}
\begin{proof}
This can be seen from Eqs.~\ref{eq:rmab_value_function} and \ref{eq:rmab_constraint}. $V$ is a maximization subject to a constraint set with size that increases with $b$. Thus increasing $b$ monotonically increases the optimal value of the optimization problem defined by $V$.
\end{proof}

\begin{theorem}\label{thm:L_monotonic}
$L_g^0(\cdot,b)$ is monotone increasing and concave in $b$.
\end{theorem}
\begin{proof}
$\frac{dL}{db} = H\sum^{T}_{k=0}\lambda_k$. Further, all $\lambda_k\ge 0$ following the Lagrangian relaxation of upper bound constraints in the \texttt{max} problem. Thus $L_g^0(\cdot,b)$ is monotone increasing in $b$. Moreover, at the optimal solution of $L_g^0(\cdot,b)$, $\lambda_k$ have values such that the optimal policies $V_n(s_n,\lambda)$ spend bH budget in expectation (this follows from $\frac{dL}{d\bm{\lambda}}$). As b increases, the optimal policies spend more in expectation, implying lower action charges $\lambda_k$, implying that $\frac{dL}{db}$ is a decreasing function, thus $L$ is concave in $b$.
\end{proof}

\begin{theorem}\label{thm:bound_L_V_difference}
$L_g^0(\cdot,b) - V_g^0(\cdot,b) < \epsilon$ where $\epsilon = (N-b)H$.
\end{theorem}
\begin{proofsketch}
For this analysis, we assume arms have the same transition functions and start state. First, we establish that the Lagrangian bound is tight at $b=0$ and $b=N$, and that the value function $V$ decouples at these values, requiring only that rewards are normalized. Then, we bound the gap between the upper bound $L(s, b)$, and some lower bound on $V(s,b)$. Note that any policy over the $N$ arms is necessarily a lower bound on the coupled optimal $V(s,b)$. It is convenient then to analyze the policy which acts on the same $b$ arms each round, consisting of per-arm value functions $V_n(s, b=1)$ which act every round and $V_n(s, b=0)$ which never act. The Lagrange bound is itself everywhere upper bounded by $N V_n(s,b=1)$. Then, for a given value of $b$, exactly $b$ terms cancel (via same transition function and start state assumption) between the lower and upper bound, and $N-b$ terms of $V_n(s,b=1)$ remain. Finally, given normalized rewards, $V_n(s,b=1) \le H$, giving $\epsilon$.
\end{proofsketch}

\subsection{Solving the Equitable Objectives}
First, note that the equitable optimization framework of Eq.~\ref{eq:general_equity_objective} is a natural bilevel optimization problem, in which the outer loop seeks to allocate budgets amongst groups and the inner loop computes the optimal policy and value function within a group, given some budget $b_g$. Such optimization problems can quickly become computationally intensive if either the inner or outer loop is inefficient. Fortunately, the inner loop is itself a well-studied problem (i.e., solving traditional RMABs), and so efficient algorithms are available for solving it, namely by solving Eq.~\ref{eq:lagrange_objective} or following the Whittle index policy \citep{whittle1988restless}. However, the outer optimization requires new techniques in the context of RMABs and our objectives. Most important is the structure of the objective with respect to the decision variables $b_g$. Towards this end, we showed in Theorems~\ref{thm:V_monotonic} and \ref{thm:L_monotonic} that $V$ and $L$ are monotonically increasing in $b$. 

\subsubsection{Solving MMR}
\label{section:algorithm_mmr}
For the \texttt{maximin} objective, monotonicity directly implies the optimality of a greedy water filling approach which iteratively assigns an additional unit of budget to the group with the lowest average value function until the budget is exhausted. The simplicity of the approach leads to remarkable computational efficiency for a problem that was previously otherwise intractable. 



\subsubsection{Solving MNW}
\label{section:algorithm_mnw}
For designing an algorithm for MNW, it is more convenient to view the equivalent log form of the objective:
\begin{equation}\label{eq:mnw_objective_log}
        \max_{\bm{b}_g} \sum_{g \in \mathcal{G}} \log \left[V^0_g(\bm{s}^0_g, \bm{b}_g)\right]
\end{equation}
With the concavity and monotonicity established in Theorem~\ref{thm:L_monotonic}, and maintained by a $\log$ transformation, it is clear that Eq.~\ref{eq:mnw_objective} can be solved optimally with a greedy approach. Specifically, we greedily assign additional units of budget to groups which achieve the maximum difference of the \textit{logs} of the value functions, i.e., line 7 of algorithm ~\ref{alg:greedy_mnw}. 
However, the MNW objective requires special consideration for optimizing group-averaged outcomes when group sizes are unequal. The reason is that, as group sizes increase, their group value functions $V_g$ scale less quickly than the group size. To shed light on this, we define 
\textit{arm-value functions}, namely, $v^a(b)$, which capture the value function of the single arm $a$, given budget $b$, and give the following theorem.

\begin{theorem}\label{thm:mnw_composite_ineqs}
The following inequalities hold $\hspace{1mm} \forall g, \hspace{1mm} \forall a \in M^{-1}(g), \hspace{1mm} \forall b$:
\begin{equation}
v^a(b) \le V_g(b) \le |M^{-1}(g)| \max_{c\in M^{-1}(g)}v^c(b)
.
\end{equation}
\end{theorem}
\begin{proof}
There exists a composite function $h_g$ which maps the set $\mathcal{A} = \{v^a : a \in M^{-1}(g)\}$ to group-value functions $V_g$. E.g., if group $g$ has two arms $c$ and $d$, then $V_g(b) = h_g(v^c, v^d)(b)$. We are interested in the relationship of $h_g$ and one of its input arm value functions $v^a$. First, for any arm set $\mathcal{A}$, $h_g(\mathcal{A})(b) \ge v^a(b) \hspace{1mm} \forall b \hspace{1mm} \forall a \in \mathcal{A}$. That is, $h_g$ is always a monotone increasing mapping (assuming non-negative reward functions, which can be done without loss of generality). This is because for all sets $\mathcal{A}$ such that $|\mathcal{A}| > 2$, $h_g(\mathcal{A})(b)$ is lower bounded by $v^a(b) + v^c(0)$ for some $a,c \in \mathcal{A}$. On the other hand, since $v^a(b)$ are monotone increasing functions of $b$, $h_g(\mathcal{A})(b)$ is also upper bounded by $|\mathcal{A}|\max_a{v^a(b)}$, since the composite causes $b$ budget to be shared by $|\mathcal{A}|$ value functions. 
\end{proof}
Its is precisely this upper bound that causes the issue. Since $V_g(b)$ scales more slowly than the group size, the slope of the $V_g(b)$ curve decrease as $|M^{-1}(g)|$ increases. Since MNW prioritizes value functions with maximal log differences, e.g., $\log(V_g(b+1)) - \log(V_g(b))$, it equivalently prioritizes value functions with maximal log ratios $\log(V_g(b+1)/V_g(b))$, or simply $V_g(b+1)/V_g(b)$. Thus larger groups, with their smaller slopes in $b$, will be prioritized after equivalent groups of smaller size.

Adjusting for this effect is tricky at first. Since MNW is scale invariant, we cannot simply divide or multiply by group size as we did for MMR. 

How to find MNW budget allocations that lack group-size bias, but still capture potential average gains of each group? The key is in the $h_g$ composite function. While it is hard to know the general form of $h_g$ (it is a function of the MDP probabilities of all arms in its input), we conjecture: 
\begin{align}  
    h_g(\mathcal{A} | \mathcal{B})(Cb) \approx C*h_g(\mathcal{A})(b) \\
    \text{where } v \sim \mathcal{A} \hspace{1.5mm} \forall v \in \mathcal{B} \text{ and } C = (|\mathcal{A}|+|\mathcal{B}|)/\mathcal{A} \nonumber
\end{align}
In words, as the set of arms passed to $h_g$ increases, the value of $h_g$ \textit{at budget value scaled by the new group size} is roughly equivalent to $h_g$ on the original set of arms times the scale factor $C$ of the new group size, assuming new arms are sampled from the same distribution as $\mathcal{A}$. If true, this would allow circumventing the group size bias problem of MNW as follows:
(1) re-sample arms in smaller groups until all groups are the same size (2) compute group value functions (3) solve for the MNW allocation (4) re-scale the budget allocations based on group size. In all experiments, we find the conjecture is well supported. We compare this re-scaling approach to a naive version of MNW which ignores group sizes, and find that it both (1) creates dramatic improvements in the equity and efficiency over the naive approach and (2) achieves in the new group-average setting the so-desired \textit{balance} of equity and efficiency that MNW objectives enjoy in other problems. 


\subsubsection{Complexity}
The full algorithms for computing the optimal objectives for MMR and MNW are given in Algs.~\ref{alg:water_filling_mmr} and ~\ref{alg:greedy_mnw}, respectively. The complexity of both algorithms is $\mathcal{O}(B|\mathcal{G}|\mathcal{C}_{\textsc{Inner}})$, where $\mathcal{C}_{\textsc{Inner}}$ is the complexity of 
the inner optimization. Note that \textsc{InnerOpt} is a subroutine to solve the inner optimization problem. To compute solutions that satisfy the bound in Thm.~ \ref{thm:bound_L_V_difference} we can directly solve Eq.~\ref{eq:lagrange_objective} with a linear program, which has approximate quadratic complexity in the number of variables \citep{jiang2020faster}, i.e., $\mathcal{O}(N^2|\mathcal{S}|^2H^2)$. Alternatively, since Whittle index solutions are widely used, and can be efficiently computed with binary search of complexity $\mathcal{O}(N|\mathcal{S}|H\log(\frac{1}{\gamma}))$ ($\gamma$ is desired precision) \citep{killian2022restless}, to improve the adoptability of our approach, we give a subroutine \textsc{WhittleToLagrange} in Alg.~3 in the appendix, which can also solve \textsc{InnerOpt}. This has complexity $\mathcal{O}(N\log(N)|\mathcal{S}|H)$, and is what we use in experiments. Finally, to take actions in the simulation environment, we take the relative budget allocations output by Algs.~\ref{alg:water_filling_mmr} and ~\ref{alg:greedy_mnw} for each group $g$, and act on the arms with the top $b_g$ Whittle indexes.

\begin{algorithm}[t]
\SetAlgoLined
\KwData{$\mathcal{G}, B, \bm{s}, h$}
$\bm{b} = 0$ \DontPrintSemicolon \tcp*{$|\mathcal{G}|$-length vector, of budgets}
\For(\commentfont{// Initialize}){$g \in \mathcal{G}$}{
    $L(\bm{s}_g, b_g)$ = \textsc{InnerOpt($g, b_g, \bm{s}_g, h$)} $/ |M^{-1}(g)|$ \\
}
\For{$b \in [1,...,B]$}{
    $g^\star = $ \textsc{ArgMin($L(\bm{s}_g, b_g)$)} \\ 
    $b_{g^\star} \pluseq 1$ \\
    $L(\bm{s}_{g^\star}, b_{g^\star})$ = \textsc{InnerOpt($g^\star, b_{g^\star}, \bm{s}_{g^\star}, h$)} $/ |M^{-1}(g)|$ \\
}
\Return $L$, $\bm{b}$
\caption{ERMAB Water Filling: Maximin Reward}
\label{alg:water_filling_mmr}
\end{algorithm}

\begin{algorithm}[t]
\SetAlgoLined
\KwData{$\mathcal{G}, B, \bm{s}, h$}
$\bm{b} = 0$ \DontPrintSemicolon \tcp*{$|\mathcal{G}|$-length vector, of budgets}
$\theta = \max_g\{|M^{-1}(g)|\}$ \\
\For(\commentfont{// Initialize}){$g \in \mathcal{G}$}{
    \textsc{UpSample}($g$,$\theta$)\DontPrintSemicolon \tcp*{Resample arms, until $g$ has size $\theta$}
    $L_0(\bm{s}_g, b_g)$ = \textsc{InnerOpt($g, b_g, \bm{s}_g, h$)} \\
    $L_1(\bm{s}_g, b_g)$ = \textsc{InnerOpt($g, b_g+1, \bm{s}_g, h$)} \\
    $L_{\Delta}(\bm{s}_g, b_g)$ = $\log(L_1(\bm{s}_g, b_g)) - \log(L_0(\bm{s}_g, b_g))$
}
\For{$b \in [1,...,B]$}{
    $g^\star = $ \textsc{ArgMax($L_{\Delta}(\bm{s}_g, b_g)$)} \\ 
    $b_{g^\star} \pluseq 1$ \\
    $L_0(\bm{s}_{g^\star}, b_{g^\star})$ = $L_1(\bm{s}_{g^\star}, b_{g^\star})$ \\
    $L_1(\bm{s}_{g^\star}, b_{g^\star})$ = \textsc{InnerOpt($g^\star, b_{g^\star}+1, \bm{s}_{g^\star}, h$)} \\
    $L_{\Delta}(\bm{s}_{g^\star}, b_{g^\star})$ = $\log(L_1(\bm{s}_{g^\star}, b_{g^\star})) - \log(L_0(\bm{s}_{g^\star}, b_{g^\star}))$
}
\textsc{Rescale}($\bm{b}, \mathcal{G}, \theta$) \DontPrintSemicolon \tcp*{Rescale budgets proportional to original group size}
\Return $L$, $\bm{b}$
\caption{ERMAB Greedy: Max Nash Welfare}
\label{alg:greedy_mnw}
\end{algorithm}

\subsection{Generality and Extensibility}
There are three important points. First is that the ERMAB framework in Eq.~\ref{eq:general_equity_objective} allows for any general function $f$, making extensions of our work to a broader class of equity functions a clear conceptual next step. Second, we bound $V_g$ with the Lagrange relaxation $L_g$ to enable computational feasibility. However, if one were to identify either a tighter bound on $V_g$ or a bound that is even more computationally convenient, one could simply pass their knew bound computation as the \textsc{InnerOpt} with to Algs.~\ref{alg:water_filling_mmr} and \ref{alg:greedy_mnw}, and they would proceed essentially the same.
Finally, the \textsc{WhittleToLagrange} subroutine we provide is a key boost to the adoptability of our approach, since it provides a lightweight method for converting the state of the art policy that RMAB planners may already have implemented, inherently incorporating any additional efficiencies or specializations they may have developed for their settings.


\section{Results}

We provide simulations across a range of parameter settings and environments. Our new policies provide significant boosts to equity with only modest reductions in utility. Code is available at \url{https://github.com/google-research/socialgood/tree/equitable-rmab}

\paragraph{Policies.}
We compare against the following baseline policies. \textbf{No Action}, which never acts. \textbf{Rand}, which selects $B$ arms to act on each round. \textbf{Opt}, which is the utility maximizing state-of-the-art RMAB Whittle index policy.
Our new algorithms are as follows. \textbf{MMR} solves for budget allocations using Alg.~\ref{alg:water_filling_mmr}, then follows a restricted Whittle index policy within each group's respective budget. \textbf{MNW-EG} solves for budget allocations using Alg.~\ref{alg:greedy_mnw}, then follows a restricted Whittle index policy within each group's respective budget. \textbf{MNW} solves for budget allocations using Alg.~\ref{alg:greedy_mnw} without the $\textsc{UpSample}$ and $\textsc{Rescale}$ steps, so it is subject to the group size bias problem described in section \ref{section:algorithm_mnw}. 

\paragraph{Domains.}
First, we design a \textbf{Synthetic} domain which highlights the key characteristics of each of the algorithms. It contains five groups. Arms in each of the first three groups (A, B, C) respond to intervention with slightly decreasing magnitude. This will cause the utility-maximizing policy to over-exploit arms in groups A and B, at the expense of group C, creating inequity. Arms in groups D and E have little response to intervention, which will cause {MMR} to over-allocate to these groups in the presence of larger budgets, creating inefficiencies. Finally, group C is size 5\%$N$, whereas other groups are size 20-25\%$N$, which will cause the naive {MNW} to over-allocate to group $C$, due to the group-size bias issue. Each arm has 2-states, and the total $N=100$. Full details are in the appendix.

Next, we consider the publicly available \textbf{Maternal Health} environment from \citep{killian2022restless}, which captures engagement behavior of mothers in an automated telehealth program. The planner's goal is to select $B$ listeners each week for intervention to boost their engagement. There are three states per arm: Self-motivated, Persuadable, and Lost Cause. There is a 1:1:3 split of arms with high, medium, and low probability of increasing engagement upon intervention. We mirror these split sizes to create three groups of size 20\%, 20\%, 60\%, where the large group is a parameter that we vary. More details are in the appendix.

\input{img/transition_diagram_no_mem.tex}

\begin{figure}[b]
    \centering
    \includegraphics[width=\columnwidth]{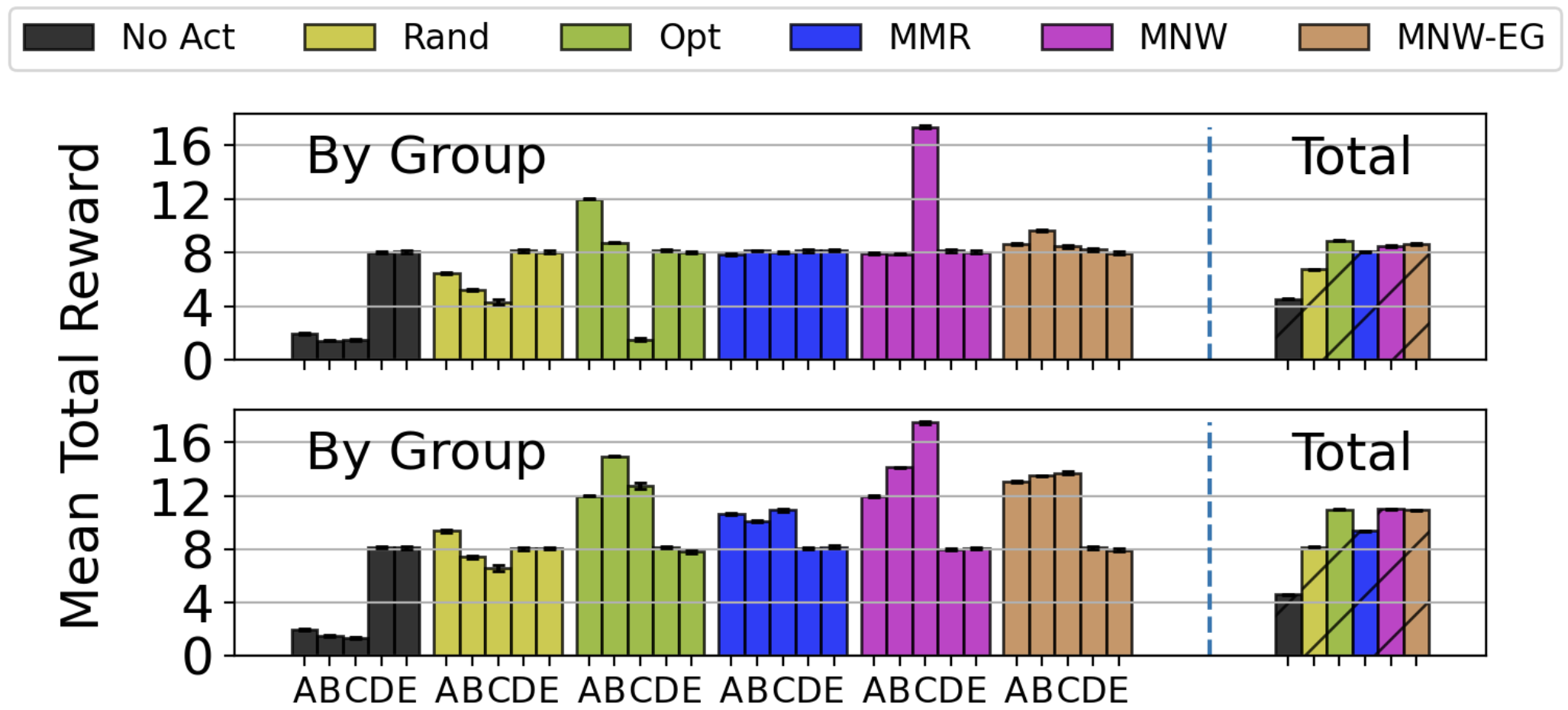}
    \caption{Synthetic domain results. Left shows group results, right shows total reward over all groups, per policy. $N=100$. Top row has $B=20$, bottom row has $B=33$. Groups A, B, and C, respond to intervention, but groups D and E have no response to intervention (MMR over-allocates to D and E). Group C has 5 arms, and all other groups have 20-25 arms (MNW over allocates to C).}
    \label{fig:results:counterexample}
\end{figure}

\begin{figure*}[ht]
    \centering
    \includegraphics[width=0.84\textwidth]{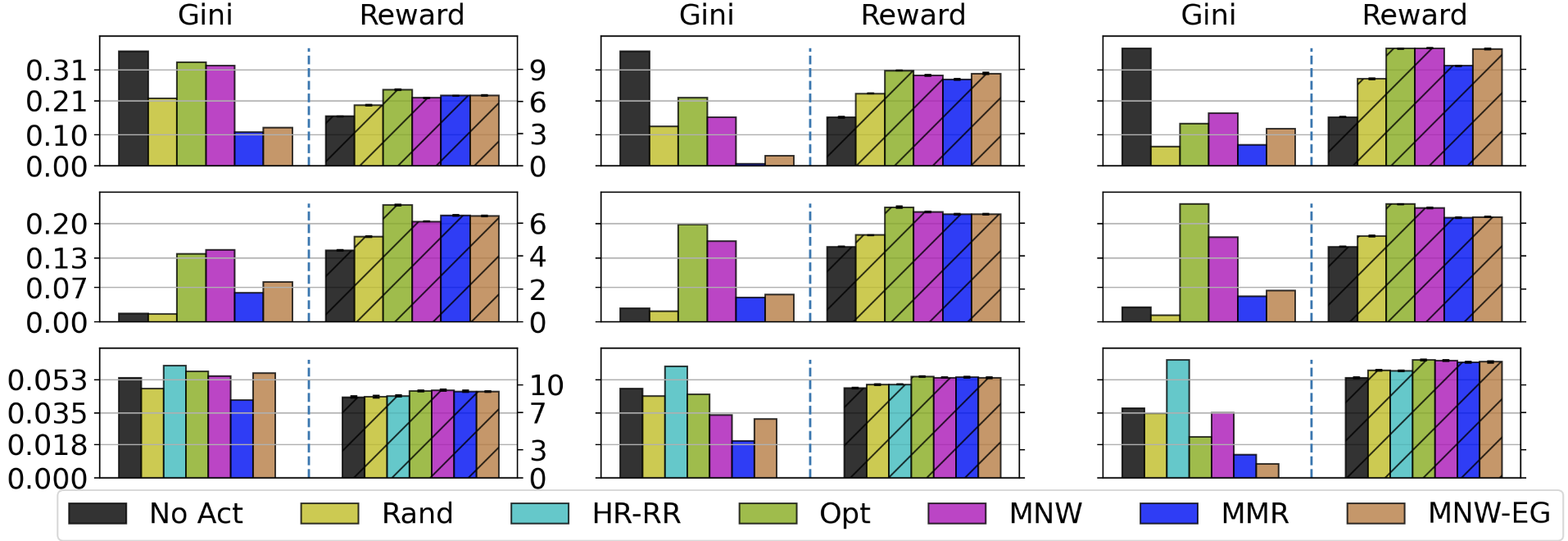}
    \caption{Gini: Lower is better. Reward: Higher is better. Top row: Synthetic. $N=100$ arms. Left to right varies budget in $[10, 20, 33]$. Middle row: Maternal Health. $N=200$, $B=60$. Left to right varies which of 3 groups is the large group (60\%N). Bottom row: Digital Diabetes. $N=300$, $B=75$ Left to right varies $\alpha$ in $[0, 0.5, 1.0]$, the weight on the engagement reward vs.~health reward. Across all settings MMR brings the most equity, while sacrificing efficiency in some cases. MNW-EG strikes a better balance, especially demonstrated in the top row.
    }
    \label{fig:results:banner_offline}
\end{figure*}

Finally, the motivation of this work is digital health care, which recently gained wide popularity \citep{onduo,vida,welldoc,mysugr}. Such programs help patients manage chronic conditions. 
Planners have dual objectives of maximizing engagement and a health outcome of patient cohorts. To build such a decision model, we consider digital care for diabetes, a disease that impacts an estimated 30 million Americans \citep{cdcDiabetesStats}. We name this domain \textbf{Digital Diabetes}.
The goal of digital diabetes care is to reduce HbA1c, a blood sugar biomarker, to a healthy range of $< 8$. However, planners must jointly reason about patients' long-term \textit{engagement} behavior to deliver effective HbA1c interventions, e.g., interventions cannot be delivered to patients who have dropped out. To capture these dynamics we construct the model shown in Fig.~\ref{fig:transition_diagram_simple}, which has a two-dimensional state space: $s_E \in \{$Engaged, Maintenance, Dropout$\}$ for engagement and $s_C \in \{\texttt{HbA1c} <8, 
\texttt{HbA1c} \ge 8\}$ for HbA1c. Engaged patients have higher probability of reaching better health states if intervened, and if a patient reaches dropout, they stay forever, mirroring real-world mechanisms. Finally, The reward $r(s_E \in \{$Engaged, Maintenance$\}) = 1$ and $r(s_E =$ Dropout$) = 0$. Similarly, the reward $r(s_C = $ HbA1c$<8)=1$ and $r(s_C = $ HbA1c$\ge 8)=0$. Given the two dimensions of reward, the planner must choose a parameter $\alpha \in [0,1]$ which determines the relative weight on the engagement and health objectives, i.e., $r(\bm{s}) = \alpha r(s_E) + (1 - \alpha)r(s_C)$.

To instantiate the model, we aggregate publicly available data: (1) on diabetes progression from IBM's widely-used MarketScan Database \citep{marketscan2021} and (2) on health program disengagement from statistics published by the National Diabetes Prevention Program (NDPP) lifestyle change program \citep{cannon2020retention}. We consider groups for three age ranges (30-44, 45-54, 55-64) and two sexes (man, woman), creating six total groups. Complete details of the model and data are in the appendix. Finally, we include a domain-specific baseline \textbf{HR-RR} which prioritizes high-A1c patients who were acted on least recently, a rough proxy for current practice in digital health programs.


\paragraph{Experiments. } All experiments were run for horizon $H=20$ over 25 random seeds. We report total reward averaged over arms or group size. Gini indexes are reported as described in section~\ref{sec:metrics}. First are results from the Synthetic domain with $N=100$ in Fig.~\ref{fig:results:counterexample}. Left of the dashed line shows average outcomes by group, and right shows the average outcome over all $N=100$ arms. Each colored bar corresponds to a policy. The top row has budget $B=20$ and the bottom row has $B=33$. For both budgets, the No Act and Rand policies are both inefficient and inequitable. Conversely, the Opt policy is always the most efficient, but there are wide gaps between the outcomes of various groups, especially groups A and C in the top row. On the other hand, MMR and MNW-EG find alternative policies that
sacrifice little efficiency, while producing far more balanced outcomes. We also see that the naive MNW significantly overallocates to the smallest group C, creating unintentional disparity. This motivates our more nuanced MNW-EG algorithm. On the bottom row, with more budget available, MMR produces worse penalties to efficiency, since it uncompromisngly focuses resources on groups D and E which have little benefit from intervention. On the other hand, MNW-EQ adeptly navigates this tradeoff, maintaining efficiency by allocating resources instead to groups A, B and C, and still equally balancing outcomes among those three groups. This is a key desirable property of the new approach.

In Fig.~\ref{fig:results:banner_offline}, we show similar results across all domains. Each row corresponds to Synthetic, Maternal Health, and Digital Diabetes, respectively. On each plot, left of the dashed line shows the Gini index of group outcomes (lower is better), and right shows the average total reward over all arms (higher is better). For the top row, we vary budget from left to right, $B = [10, 20, 33]$. The same conclusions hold as for Fig.~\ref{fig:results:counterexample}, with the addition of the Gini metric to quantify inequity. We see that MMR and MNW-EQ are always more balanced than Opt, and for $B=20$, they are 20 times and 10 times for more balanced respectively, with an almost negligible drop in total performance. In the middle row (Maternal Health), we set $N=200$; left to right varies which of the three groups has size 60\%N. Across all settings, the equitable policies have a maximum of 15\% drop in efficiency, while gaining 2-4 times improvements in equity. Finally, in the bottom row (Digital Diabetes), we set $N=300$; left to right varies $\alpha = [0, 0.5, 1]$. For $\alpha$ of $0$ and $0.5$, interestingly, Opt, MMR, and MNW-EQ perform virtually the same in efficiency, but MMR distributes the outcomes 1.5 and 2.5 times more equitably. This is especially promising since $\alpha=0$ is the case in which reward corresponds directly to patient blood sugar, demonstrating the potential for this framework to produce more equitable health outcomes, the inspiration of this work. Further analysis in the appendix includes: (1) Pareto analysis over $\alpha$ and (2) resource capacity planning, demonstrating an alternative use of our approach to decide \textit{how many resources should be acquired} to help ensure, e.g., a given cohort reaches health goals equitably.

\paragraph{Conclusion.}
We make key conceptual and algorithmic contributions by introducing Equitable RMABs, and designing two algorithms for reaching equitable solutions. We hope these benefit the work of practitioners addressing resource allocation problems in real-world domains.



\begin{acknowledgements} This work was supported in part by the Army Research Office
by Multidisciplinary University Research Initiative (MURI) grant number W911NF1810208. J.A.K.~was supported by an NSF Graduate Research Fellowship under grant DGE1745303. J.A.K.~was a Student Researcher at Google and Verily for parts of the project.
\end{acknowledgements}

\bibliography{9_references}

\clearpage
\input{8_appendix}


\end{document}

%% file: 2_intro.tex
Restless multi-armed bandits (RMABs) are a sequential decision making framework in which a central planner must optimally allocate a restricted set of $b$ resources to $N$ independent control processes (arms) over time, and have been shown to be an effective model for many real world problems. As such, they have been closely studied with wide ranging applications to, e.g., 
public health program optimization \citep{mate2022field,deo2013improving},
wireless network scheduling \citep{modi2019transfer,cohen2014restless},
anti-poaching \citep{qian2016restless}, treatment optimization \citep{lee2019optimal,ayer2019prioritizing}, and more. 
The goal of an RMAB is to maximize utility over all arms. However, especially when the RMAB decision can impact human lives, pure utility maximization can lead to unbalanced outcomes.

\begin{figure}[t]
    \centering
    \includegraphics[width=\columnwidth]{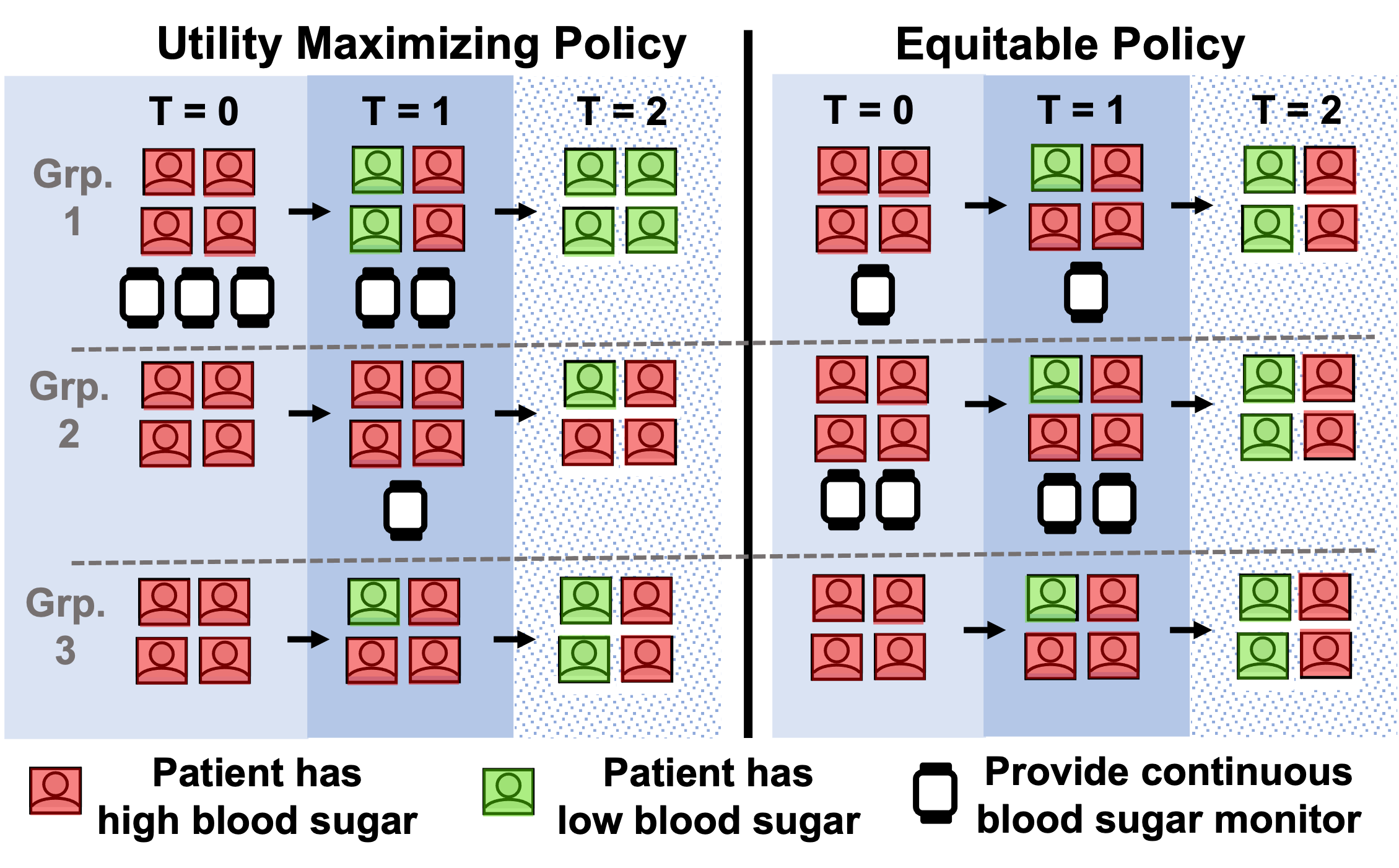}
    \caption{Utility maximizing vs.~equitable policy for resource allocation among three groups with different behavior; 3 monitors can be sent per timestep. Equitable policy balances improvements in outcomes (T=2) across groups.}
    \label{fig:concept}
\end{figure}

Consider the example visualized in Fig.~\ref{fig:concept}. A digital health program delivers care to a cohort of patients with diabetes by helping them reduce their blood sugar. The program has a limited set of expensive continuous blood sugar monitors which can be sent as interventions each month. The patient cohort has three groups: group 1 which are tech savvy and so have quick blood sugar reduction with the monitors, group 2 which have lesser response than group 1, and group 3, which are tech averse, but more self-motivated, and so have blood sugar reductions regardless of intervention.
A utility-maximizing bandit would prioritize sending monitors to group 1 to create the greatest cohort-level reduction in blood sugar. Though at first look it is a natural objective, the cohort-level gains would come at the expense of group 2, exacerbating disparities in health outcomes.

In this scenario--the motivating domain of our work--inequity across group outcomes is unacceptable; health inequities must \textit{not} be exacerbated, but rather directly addressed and rectified \citep{cdcequity}. The same principle holds in other low resource domains \citep{luss2012equitable}.
To address the imbalance of utility maximization in RMABs, some studied the notion of \textit{equality} \citep{li2022towards,herlihy2021planning}, in which groups must share a minimum probability of consideration for action. However, this still fails to guarantee balanced outcomes. In our diabetes problem, equal allocation may unnecessarily send monitors to group 3, at the expense of groups 1 and 2. To address outcome disparity, a planner must reason about what each group \textit{needs} to achieve a goal, and act accordingly. This is the notion of \textit{equity} \citep{cdcequity,luss2012equitable}. In Fig.~\ref{fig:concept} (right), equity is achieved by seeing that, to reach a healthy state, patients with high blood sugar in group 2 need monitors for roughly twice as long as analogous patients in group 1.

To the best of our knowledge, outcome-based equity in RMABs has not been studied. 
Herein, we study for the first time \emph{equitable restless multi-armed bandits} (ERMAB), which require that policies take affirmative steps to distribute resources to ensure high-performing and balanced outcomes across pre-specified \emph{groups} of arms. Note that even with known arm control dynamics, RMABs are at least PSPACE hard to solve exactly \citep{papadimitriou1999complexity}, and our equitable objectives add coupling that further increases complexity.
Therefore, this work builds the optimization principles needed to newly solve the \textit{offline} problem in which arm models are known (common practice in literature \cite{modi2019transfer,lee2019optimal}).

We cast ERMAB as a bilevel optimization: 
the inner problem seeks good policies \textit{within} groups, and the outer problem seeks equitable policies \textit{across} groups. We study ERMABs considering two equity-aligned objectives from the fairness literature: maximin reward (MMR), a conservative objective that prioritizes worst-off groups first, and maximum Nash welfare (MNW), which naturally balances equity and total utility.  Within each objective, we prove key monotonic structures, and address the problem's immense complexity by deriving algorithms from Lagrangian relaxation. For MMR, our proofs enable a fast and well-performing water-filling algorithm. For MNW, varied \textit{group sizes} complicate planning. We show specifically that MNW policies bias allocations toward \textit{small} groups, and prove this is due to the structure of the \textit{composite} function over individual arm value functions in a group. We correct for the bias in a greedy algorithm with nuance that resamples arms to simulate equalized group sizes, computes allocations, then rescales to original group sizes. Crucially, our procedure is fast and finds attractive tradeoffs between utility and equity.

We then demonstrate the importance of our equitable methods across three domains, showing that we find allocations that lead to outcomes 3-5 times more balanced across groups than previous state of the art, with minimal reductions in system-wide utility. 
Critically, in a new digital diabetes care environment we propose and instantiate with public data, we find policies that are both 3 times more balanced \textit{and} achieve 90\% of the maximum efficiency gains. Clearly equitable policies are needed, and we show they are attainable.

%% file: img/transition_diagram_no_mem.tex
\begin{figure}[t]
    \centering
\scalebox{0.6}{    
\begin{tikzpicture}[node distance=0.5cm and 0.5cm,>=stealth',auto, every place/.style={draw,minimum size=2.5cm,text width=2.5cm,align=center,node distance=0.5cm}]
    \node [place] (S1) {\textbf{\texttt{Engaged}}\\\mbox{$\textbf{\texttt{HbA1c}}\bm{\ge 8}$}};
    \node [place] (S2) [right=of S1] {\textbf{\texttt{Engaged}}\\\mbox{$\textbf{\texttt{HbA1c}}\bm{< 8}$}};
    


    \node [place] (S3) [below =of S1] {\textbf{\texttt{Maint.}}\\\mbox{$\textbf{\texttt{HbA1c}}\bm{\ge 8}$}};
    \node [place] (S4) [below =of S2] {\textbf{\texttt{Maint.}}\\\mbox{$\textbf{\texttt{HbA1c}}\bm{<8}$}};
    
    \node [place] (S5) [below =of S3] {\textbf{\texttt{Dropout}}\\\mbox{$\textbf{\texttt{HbA1c}}\bm{\ge 8}$}};
    \node [place] (S6) [below =of S4] {\textbf{\texttt{Dropout}}\\\mbox{$\textbf{\texttt{HbA1c}}\bm{< 8}$}};
    
    \path[->] (S1) edge [bend left=15] node {} (S2);
    \path[->] (S2) edge [bend left=15] node {} (S1);
    \path[->] (S1) edge [loop left] node {} ();
    \path[->] (S2) edge [loop right] node {} ();
    
    \path[->] (S1) edge [bend left=15,dashed] node {} (S3);
    \path[->] (S1) edge [bend right=15,dashed] node {} (S4);
    \path[->] (S2) edge [bend left=15,dashed] node {} (S3);
    \path[->] (S2) edge [bend left=15,dashed] node {} (S4);
    
    \path[->] (S3) edge [bend left=15,dashed] node {} (S4);
    \path[->] (S4) edge [bend left=15,dashed] node {} (S3);
    \path[->] (S3) edge [loop left,dashed] node {} ();
    \path[->] (S4) edge [loop right,dashed] node {} ();
    
    \path[->] (S3) edge [bend left=15] node {} (S1);
    \path[->] (S3) edge [bend left=15] node {} (S2);
    \path[->] (S4) edge [bend right=15] node {} (S1);
    \path[->] (S4) edge [bend left=15] node {} (S2);
    
    \path[->] (S3) edge [bend left=15,dashed] node {} (S5);
    \path[->] (S3) edge [bend right=15,dashed] node {} (S6);
    \path[->] (S4) edge [bend left=15,dashed] node {} (S5);
    \path[->] (S4) edge [bend left=15,dashed] node {} (S6);
    \path[->] (S3) edge [bend right=15] node {} (S5);
    \path[->] (S4) edge [bend right=15] node {} (S6);
    
    \path[->] (S5) edge [bend left=15,dashed] node {} (S6);
    \path[->] (S6) edge [bend left=15,dashed] node {} (S5);
    \path[->] (S5) edge [loop left,dashed] node {} ();
    \path[->] (S6) edge [loop right,dashed] node {} ();

\end{tikzpicture}
}

\caption{Transition graph for a \textbf{Digital Diabetes} arm. Bold (dotted) arrows correspond to $a=1$ ($a=0$). 
}
\label{fig:transition_diagram_simple}
\end{figure}

%% file: 8_appendix.tex
\section{Whittle To Lagrange Algorithm}

\begin{algorithm}[h]
\SetAlgoLined
\KwData{$g, b, \bm{s}, h, W(\bm{s})$}
$W^{\textsc{sort}} = $ SortDescending($W(
\bm{s})$) \\
$\lambda^\star$ = $(W^{\textsc{sort}}[b] + W^{\textsc{sort}}[b+1])/2$ \\
\For{$n \in M^{-1}(g)$}{
    $V_n(s_n, \lambda^{\star})$ = \textsc{ValueIteration($n, s_n, \lambda^\star$)} \\
}
$L = \sum_{n=1}^{N}V_n(s_n, \lambda^{\star}) + bh\lambda^\star$ \\
\Return $L$
\caption{Whittle To Lagrange}
\label{alg:whittle_to_lagrange}
\end{algorithm}

\section{Digital Diabetes Model}

The inspiration of our equitable RMAB work derived from our efforts building decision models for the delivery of limited resources in digital health care settings. Here, we give complete details of the Digital Diabetes RMAB model we constructed to model the problem and used for evaluation in the main-text experiments. \textit{The focus of our model is to study the joint engagement-health dynamics of digital health programs in the Type 2 Diabetes (T2D) context, and identify better intervention strategies.}

\subsection{Model}

To capture the joint engagement-health dynamics of digital health programs, we include a dimension for each in our state space. For the T2D domain, we also include a dimension for memory, since intervention effects have a delayed impact on the clinical state. We represent this 3-dimensional state space $\mathcal{S}$ by a three-tuple $(s_E, s_C, s_M)$, where $s_E$ captures the arm's \emph{engagement}, $s_C$ captures the arm's \emph{clinical state}, and $s_M$ is a two-length \emph{memory} vector. All dimensions of the state space are modeled as discrete, where continuous spaces are discretized via threshold rules, described next. 

The engagement dimension, $s_E$, has three states: \{\texttt{Engaged}, \texttt{Maintenance}, \texttt{Dropout}\}. A patient is \texttt{Engaged} if they received an intervention from the care team \textit{and} they responded to the team within the app in the current time period. A patient is in the \texttt{Maintenance} state if they have produced any interactions within the app, but did not respond to an intervention, if it was attempted in the current time period. A patient is in the \texttt{Dropout} state if they have not produced any interactions in the app in the current time period \textit{and} will no longer do so in any future time period (e.g., they have deleted the app). These states are chosen to capture primary high-level engagement dynamics seen in digital health programs. 

The clinical dimension, $s_C$, captures a patient's HbA1c value (equivalently A1c) via two states: \{$\texttt{A1c} < 8$, $\texttt{A1c} \ge 8$\}. This threshold was chosen to model the clinical target for app users in publicly available data, i.e., reducing their HbA1c below 8.

Finally, the memory dimension, $s_M$, is a two-length vector for recording previous values of $s_E$, so its entries can take the same values as the $s_E$ dimension.
The memory serves to implement a 3-month delay between an intervention and its impact on the clinical state. This effect is observed in data and is due to the biological nature of HbA1c progression, i.e., that it is a summary measure of the body's blood sugar over the previous 3 months \cite{cdcA1c}.
Let $s_{M_i}$ reference the $i^{\text{th}}$ entry of the zero-indexed, 2-length memory vector.

Transition dynamics are summarized below and visualized in Figs.~\ref{fig:transition_diagram}--\ref{fig:transition_diagram_clinical}.

\input{img/transition_diagram.tex}
\input{img/transition_diagram_clinical}

\textbf{Engagement Dynamics. }
The engagement model is made up of four main effects. First, each patient has their own independent probability of responding to an intervention and transitioning to the \texttt{Engaged} state from either the \texttt{Engaged} or \texttt{Maintenance} states. Second, the probability of a patient responding to an intervention if they were previously in the \texttt{Engaged} state is higher than if they were previously in the \texttt{Maintenance} state. Third, the probability of a patient transitioning to \texttt{Dropout} state is lower if the patient receives an intervention, than if they do not. Lastly, patients in the \texttt{Dropout} state will never respond to an intervention. In summary, this corresponds to four open parameters for the engagement dynamics,
$p^{\texttt{I}}_{\texttt{MtoE}}$, 
$p^{\texttt{I}}_{\texttt{EtoE}}$,
$p^{\texttt{I}}_{\texttt{MtoD}}$, and
$p^{\texttt{U}}_{\texttt{MtoD}}$,
where the superscript $I$ or $U$ denote the action. 

\textbf{Clinical Dynamics. }
There are two meaningful clinical dynamics, corresponding to the clinical evolution of patients who did and did not respond to an intervention. Specifically, we assume that patients who received and responded to an intervention (i.e., were in the \texttt{Engaged} state) will have a higher probability of transitioning to a healthy clinical state than a patient who did receive or respond to an intervention. In addition, all effects are delayed by 3 months via the memory states as described in the equations below and shown in Fig.~\ref{fig:transition_diagram_clinical}. Note that we assume that A1c progression is the same for users who were in the \texttt{Maintenance} and \texttt{Dropout} states.
We show the evolution of the clinical state $s'_C$, given the memory state $s_{M_1}$ (i.e. clinical state 3 months ago), and the current clinical state $s_C$, in Table~\ref{tab:evolution}.
\begin{table}[htpb]
    \centering
    \begin{tabular}{|p{0.28\columnwidth}|c|c|}
        \hline 
         \textbf{Evolution of clinical state $P(s^\prime_C = \textbf{\texttt{A1c}}\bm{< 8} | r, c)$} & $s_C=\textbf{\texttt{A1c }}{\bm\ge 8}$ & $s_C=\textbf{\texttt{A1c }}{\bm< 8}$ \\[12pt] \hline
         \centering $s_{M_1}=\texttt{Eng}$ & $p^{\texttt{E}}_{A1c \ge 8}$ & $p^{\texttt{E}}_{A1c < 8}$ \\[5pt]
         \centering $s_{M_1}\ne\texttt{Eng}$ & $p^{\overline{\texttt{E}}}_{A1c \ge 8}$ & $p^{\overline{\texttt{E}}}_{A1c < 8}$ \\[2pt]
         \hline
    \end{tabular}
    \caption{The table shows the evolution of clinical state $P(s^\prime_C = \textbf{\texttt{A1c}}\bm{< 8} | r, c)$ where $r$ represents the memory state $s_{M_1}$ and $c$ represents the current clinical state $s_C$.}
    \label{tab:evolution}
\end{table}
Row 1 of Table~\ref{tab:evolution} represents users who received and responded to an intervention $3$ months ago, whereas row 2 represents users who \textit{did not} receive/respond to an intervention 3 months ago.



Note that this requires estimating only 4 parameters for clinical progression, i.e., $ p^{\texttt{E}}_{A1c \ge 8}, p^{\texttt{E}}_{A1c < 8}, p^{\overline{\texttt{E}}}_{A1c \ge 8}, p^{\overline{\texttt{E}}}_{A1c < 8}$, all of which encode the probability of having A1c less than 8 in the next month.

\textbf{Memory Dynamics. }
The memory dimension is a sliding window to record the engagement state of the previous three months:
$$P(s^\prime_{M_0} = s_E, s^\prime_{M_1} = s_{M_0} | s_{E}, s_{M_0}) =1
$$
Finally, note that arrows in Figs.~\ref{fig:transition_diagram} and~\ref{fig:transition_diagram_clinical} represent joint engagement-clinical-memory transition probabilities. These are obtained by multiplying the engagement, clinical, and memory transition rules.

\textbf{Observability. }
We consider a fully observable model in the experiments work, but accounting for partial observability of the health state is a key area of future study. 

\subsection{Rewards}
We assign rewards based on the current state of each patient, and represent them as $R(\bm{s})$. In general, our objective is to jointly boost engagement and clinical state. To capture that objective, we define rewards for each state dimension $d$ independently, i.e., $r_d(s)$ as:
\begin{align}
r_E(s_E = \texttt{D.O.}) = 0,\hspace{1mm} \nonumber\\
r_E(s_E = \texttt{Maint}) = 1,\hspace{1mm} \\
r_E(s_E = \texttt{Eng}) = 1 \nonumber \\
r_C(s_C = \textbf{\texttt{A1c}}\bm{> 8}) = 0, \nonumber\\
r_C(s_C = \textbf{\texttt{A1c}}\bm{< 8}) = 1
\end{align}
The reward for a patient's full state is then computed as $R([s_E, s_C,\\ s_M]) = \alpha r_E(s_E) + (1-\alpha)r_C(s_C)$. Thus the parameter $\alpha$ represents the relative weight on the engagement reward, is can be tuned based on the planner's desired objective.

\section{Digital Diabetes Data Details}

\noindent\textbf{MarketScan. }To derive baseline statistics on clinical evolution, we utilize IBM's widely-used Truven Health MarketScan Commercial Database \cite{marketscan2021}, a convenience sample of medical insurance claims from privately insured patients in the United States over the years 2018 to 2020, which includes measurements of A1c.
We consider users enrolled for more than 6 months that have T2D only, i.e., excluding those with hypertension, depression, 
heart failure,  or cancer. We then group users by age, sex, and starting A1c to derive statistics per group on monthly A1c change (full details in appendix).
These provide values of $p^{\overline{\texttt{E}}}_{A1c \ge 8}$ and $p^{\overline{\texttt{E}}}_{A1c < 8}$ of approximately $7.5$\% and $0.5\%$, respectively, with about $1\%$ variation across groups.
MarketScan dataset is publicly accessible and provides a reasonable estimate for the background rate of A1c change for users not in a specific digital health program, but still taking steps to manage their T2D on their own. It provides a conservative baseline for our experiments.

For the engagement dynamics, statistics on monthly dropout rates by demographic groups from \emph{digital} health programs are not readily available. Therefore, we use age and sex-based monthly dropout statistics published by the National Diabetes Prevention Program (NDPP) lifestyle change program, primarily made up of in-person meetings \cite{cannon2020retention}.
With monthly dropout rates near 10\%, this again forms a reasonable conservative baseline for experiments, serving as a proxy for patients' willingness to engage with T2D-related ongoing behavior change coaching. 
These statistics populate 
$p^{\texttt{U}}_{\texttt{MtoD}}$ in our model, with about $4\%$ variation between groups. 

The remaining parameters require estimates from digital health program data which is not readily available publicly. Thus we make the following assumptions to instantiate their values.
For $p^{\texttt{E}}_{A1c \ge 8}$ and $p^{\texttt{E}}_{A1c < 8}$, i.e., the clinical probabilities of patients who received and responded to intervention, the patients in age ranges 30-44, 45-54, and 55-64 receive 25\%, 50\%, and 75\% boost in their clinical probability of transitioning to A1c<8, respectively. We found that this leads to clinical trajectories in line with one published observational study of a digital diabetes management program \cite{bergenstal2021remote}, and included age-based variation to align with variation observed in NDPP's monthly dropout statistics. For
$p^{\texttt{I}}_{\texttt{EtoE}}$ and 
$p^{\texttt{I}}_{\texttt{MtoD}}$, we assign values of $99\%$ and $3\%$, respectively, encoding an assumption that patients are more likely to stay in the program if intervened and/or if already engaged. For
$p^{\texttt{I}}_{\texttt{MtoE}}$, we assign values about value of 75\%, but with the same group variation as was present in the data for NDPP's dropout statistics. Finally, we set the probability of observing the clinical state of a patient in the maintenance state, i.e., $q^{\texttt{Obs}}_{\texttt{Maint}}$ to 30\%, in line with statistics from MarketScan.

\subsection{The Marketscan dataset.} 
We consider users enrolled for more than 6 months that have T2D only, i.e., excluding those with hypertension, depression,
heart failure,  or cancer.
We filter to patients who had at least two measurements separated by more than 28 days, which comprise 35k patients.
We then compute each patient's monthly A1c change based on the difference between their A1c readings separated the furthest in time, to reduce noise inherent in more adjacent readings. We then group patients by starting A1c (greater or less than 8) and compute the average and standard deviation of starting A1c and monthly A1c change per group. Finally, we convert those statistics to monthly probabilities $p^{\overline{\texttt{E}}}_{A1c \ge 8}$ and $p^{\overline{\texttt{E}}}_{A1c < 8}$ by computing the number of patients who would reach or remain at an A1c of 8 or less assuming a normal distribution for (1) starting A1c per group and (2) monthly A1c change per group. These provide values of $p^{\overline{\texttt{E}}}_{A1c \ge 8}$ and $p^{\overline{\texttt{E}}}_{A1c < 8}$ of approximately $7.5$\% and $99.5\%$, respectively, with total variation across groups about 1 percentage point. The final table of parameters is given in Table~\ref{table:marketscan}.


\section{Additional Domain Details}

\subsection{Synthetic} 
The \textbf{Synthetic} domain highlights the key characteristics of each of the algorithms. All arms have 2 states. The reward for states are $r(0) = 0$ and $r(1) = 1$. It contains five groups. Arms in each of the first three groups [A, B, C] respond to intervention with slightly decreasing magnitude. Specifically, their transition probabilities $p(s, a, s^\prime)$ are as follows. For all groups in [A, B, C], $p(0, 0, 1) = 0.05$. For group A, $p(1, 0, 1) = 0.35$, $p(0, 1, 1) = 0.99$, and $p(1, 1, 1) = 0.99$. For group B, $p(1, 0, 1) = 0.10$, $p(0, 1, 1) = 0.95$, and $p(1, 1, 1) = 0.95$. For group C, $p(1, 0, 1) = 0.05$, $p(0, 1, 1) = 0.90$, and $p(1, 1, 1) = 0.90$. For groups D and E, $p(0, 0, 1) = 0.4$, $p(1, 0, 1) = 0.4$, $p(0, 1, 1) = 0.4$, and $p(1, 1, 1) = 0.4$. The percentage of arms in groups [A, B, C, D, E] are [0.25, 0.25, 0.05, 0.25, 0.20]. 

\subsection{Maternal Health} 
This is the publicly available \textbf{Maternal Health} environment from \citep{killian2022restless}, which captures engagement behavior of mothers in an automated telehealth program. The planner's goal is to select $b$ listeners each week for intervention to boost their engagement. There are three states per arm: Self-motivated, Persuadable, and Lost Cause. There is a 1:1:3 split of arms with high, medium, and low probability of increasing engagement upon intervention, corresponding to types A, B, and C in \citet{killian2022restless}, respectively.

Let Self-motivated be state 0, Persuadable be state 1, and Lost Cause be state 2. For all arms, the rewards are $R(0) = 1$, $R(1) = 0.5$, and $R(2) = 0$. For all arms, transitions are only possible between adjacent states, i.e., $p(0, \cdot, 2) = p(2, \cdot, 0) = 0$. For group A $p(0, 0, 0) = 0.5$, $p(0, 1, 0) = 0.5$, $p(1, 0, 2) = 0.75$, $p(1, 1, 0) = 0.75$, $p(2, 0, 2) = 0.60$, $p(2, 1, 2) = 0.60$. For group B $p(0, 0, 0) = 0.5$,  $p(0, 1, 0) = 0.5$, $p(1, 0, 2) = 0.60$, $p(1, 1, 0) = 0.40$, $p(2, 0, 2) = 0.60$, $p(2, 1, 2) = 0.60$. For group C $p(0, 0, 0) = 0.5$,  $p(0, 1, 0) = 0.5$, $p(1, 0, 2) = 0.60$, $p(1, 1, 0) = 0.25$, $p(2, 0, 2) = 0.60$, $p(2, 1, 2) = 0.60$. The percentage of arms in each group is $0.2N$ for two of the groups, and the other is $0.6N$ -- the group that is size $0.6N$ is varied in experiments (across columns of row 2 in the main text results). To induce additional variation in the experimental results, for each random seeded experiment, probabilities for each entry for each arm are sampled using its group's $p(\cdot)$ values as the mean, and a standard deviation of $0.2$ times the mean's min absolute distance from 0 or 1 (to scale it appropriately as it approaches 0 or 1). This is detailed in the function $\textsc{sample\_prob}$ in the code.

\section{Additional Results: Digital Diabetes}
We include additional baseline \textbf{HR-Rand} which acts randomly among all high-A1c arms each round (excluded from main text since it performs about the same as HR-RR). Fig.~\ref{fig:appendix:pareto_60} shows the Pareto curve for $B=60$, varying $\alpha$. This underscores the robustness of the gains of RMAB planning across values of $\alpha$. It also demonstrates how planners could tune their objectives, e.g., by choosing $\alpha=0.25$ to get a roughly 10\% boost to engagement with negligible reductions in A1c. 
Fig.~\ref{fig:appendix:capacity-alpha0.5} shows another potential use of our system for digital health planners, namely capacity planning. The dashed line shows a desired final aggregate state of the system. Where each policy line intersects with the dashed line indicates the estimated monthly intervention budget that would be needed to achieve that final state. This gives planners a tool to take estimated engagement-health dynamics data for a prospective cohort, and output an estimated number of support resources needed to help that cohort reach a desired health outcome. Moreover, we see with this analysis that with optimized RMAB policies, one can achieve the example desired final state with roughly 50\% fewer resources than via current intervention assignment rules.

\begin{figure}[H]
    \centering
    \includegraphics[width=\columnwidth]{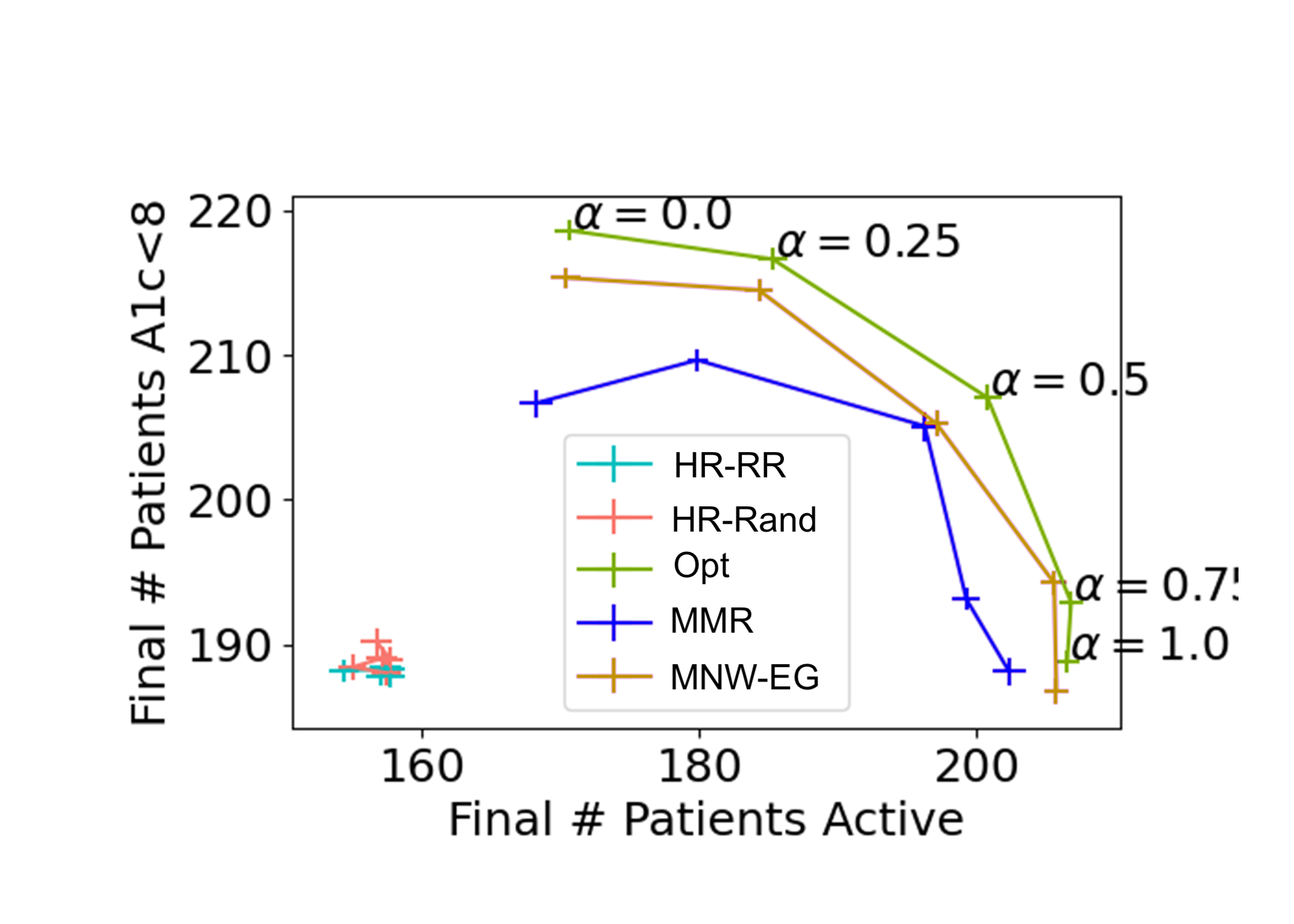}
    \caption{Pareto curve showing engagement vs.~health, with $N=300$ and $B=60$, Digital Diabetes domain.}
    \label{fig:appendix:pareto_60}
\end{figure}




\begin{figure}[H]
    \centering
    \includegraphics[width=\columnwidth]{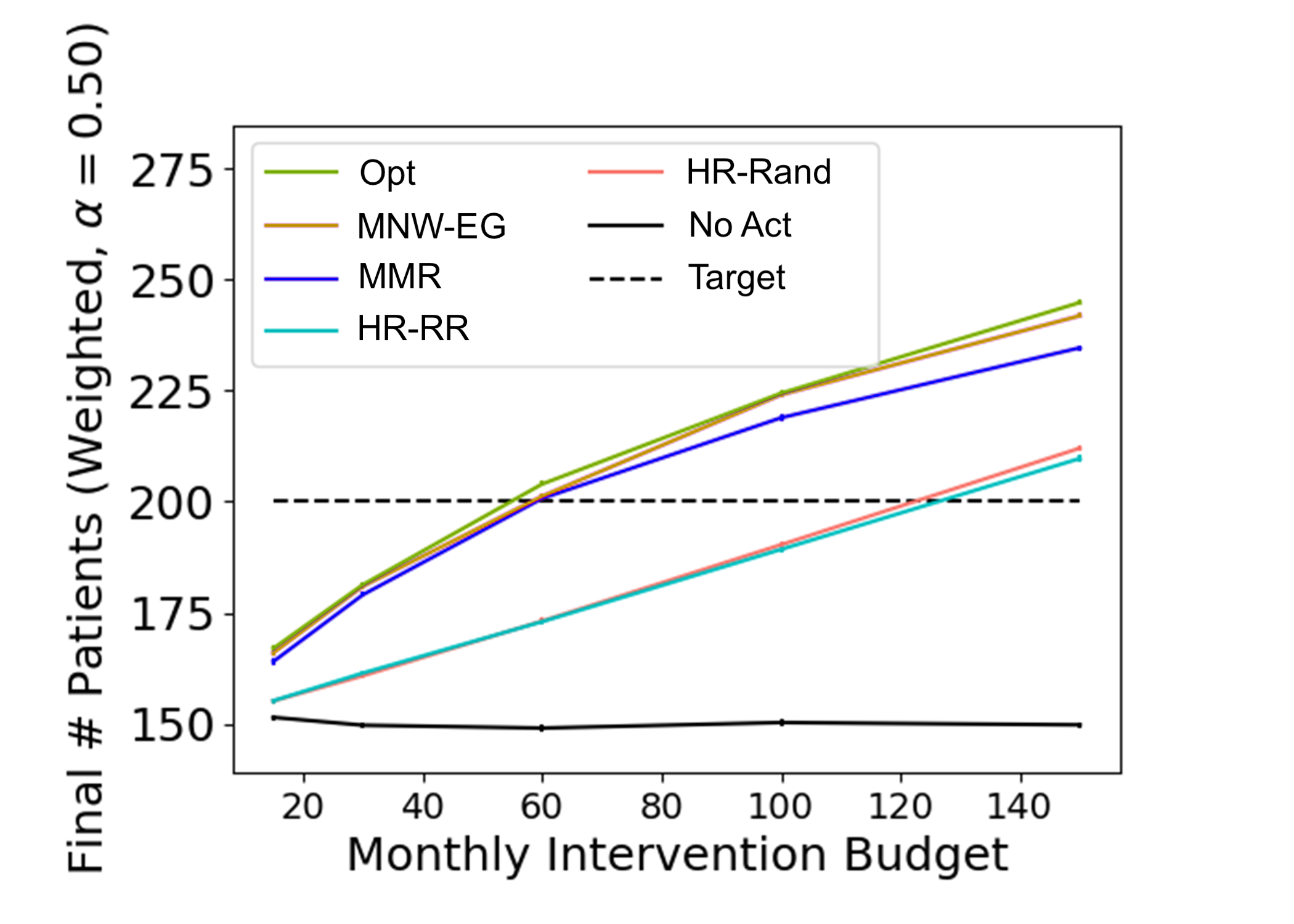}
    \caption{Capacity planning for the Digital Diabetes domain with $\alpha=0.5$, $N=300$}
    \label{fig:appendix:capacity-alpha0.5}
\end{figure}




\begin{table*}[t]\centering
\caption{MarketScan experiment parameters.}
\begin{tabularx}{\textwidth}
{rrrrrrrrrrrl}
\toprule
 Group &  $p^{I}_{MtoE}$ &  $p^{I}_{MtoD}$ &  $p^{I}_{EtoE}$ &  $p^{U}_{MtoD}$ &  $p^{\overline{\texttt{E}}}_{A1c \ge 8}$ &  $p^{\overline{\texttt{E}}}_{A1c < 8}$ &  $p^{\texttt{E}}_{A1c \ge 8}$ &  $p^{\texttt{E}}_{A1c < 8}$ &  frac &  sex &   age \\
\midrule
     0 &     0.560 &      0.03 &      0.99 &     0.122 &       0.071 &       0.992 &     0.089 &     0.994 & 0.175 &    1 & 30-44 \\
     1 &     0.783 &      0.03 &      0.99 &     0.093 &       0.074 &       0.990 &     0.111 &     0.995 & 0.150 &    1 & 45-54 \\
     2 &     0.907 &      0.03 &      0.99 &     0.077 &       0.080 &       0.993 &     0.140 &     0.998 & 0.200 &    1 & 55-64 \\
     3 &     0.560 &      0.03 &      0.99 &     0.122 &       0.069 &       0.992 &     0.087 &     0.994 & 0.150 &    2 & 30-44 \\
     4 &     0.783 &      0.03 &      0.99 &     0.093 &       0.070 &       0.993 &     0.104 &     0.996 & 0.125 &    2 & 45-54 \\
     5 &     0.907 &      0.03 &      0.99 &     0.077 &       0.085 &       0.995 &     0.148 &     0.999 & 0.200 &    2 & 55-64 \\
\bottomrule
\label{table:marketscan}
\end{tabularx}
\end{table*}

\section{Proofs}
\input{3_proofs}

%% file: img/transition_diagram.tex
\begin{figure}
    \centering
\scalebox{0.7}{    
\begin{tikzpicture}[node distance=0.5cm and 0.5cm,>=stealth',auto, every place/.style={draw,minimum size=3cm,text width=2.5cm,align=center,node distance=0.5cm}]
    \node [place] (S1) {$s_E=\textbf{\texttt{Eng.}}$\\\mbox{$s_C=\textbf{\texttt{A1c}}\bm{\ge 8}$}\\$s_M=\bm{M}$};
    \node [place] (S2) [right=of S1] {$s_E=\textbf{\texttt{Eng.}}$\\\mbox{$s_C=\textbf{\texttt{A1c}}\bm{< 8}$}\\$s_M=\bm{M}$};
    


    \node [place] (S3) [below =of S1] {$s_E= \textbf{\texttt{Maint.}}$\\\mbox{$s_C=\textbf{\texttt{A1c}}\bm{\ge 8}$}\\$s_m=\bm{M}$};
    \node [place] (S4) [below =of S2] {$s_E=\textbf{\texttt{Maint.}}$\\\mbox{$s_C=\textbf{\texttt{A1c}}\bm{<8}$}\\$s_M=\bm{M}$};
    
    \node [place] (S5) [below =of S3] {$s_E=\textbf{\texttt{D.O.}}$\\\mbox{$s_C=\textbf{\texttt{A1c}}\bm{\ge 8}$}\\$s_M=\bm{M}$};
    \node [place] (S6) [below =of S4] {$s_E=\textbf{\texttt{D.O.}}$\\\mbox{$s_C=\textbf{\texttt{A1c}}\bm{< 8}$}\\$s_M=\bm{M}$};

    
    
    
    \path[->] (S1) edge [bend left=15] node {} (S2);
    \path[->] (S2) edge [bend left=15] node {} (S1);
    \path[->] (S1) edge [loop left] node {} ();
    \path[->] (S2) edge [loop right] node {} ();
    
    \path[->] (S1) edge [bend left=15,dashed] node {} (S3);
    \path[->] (S1) edge [bend right=15,dashed] node {} (S4);
    \path[->] (S2) edge [bend left=15,dashed] node {} (S3);
    \path[->] (S2) edge [bend left=15,dashed] node {} (S4);
    
    \path[->] (S3) edge [bend left=15,dashed] node {} (S4);
    \path[->] (S4) edge [bend left=15,dashed] node {} (S3);
    \path[->] (S3) edge [loop left,dashed] node {} ();
    \path[->] (S4) edge [loop right,dashed] node {} ();
    
    \path[->] (S3) edge [bend left=15] node {} (S1);
    \path[->] (S3) edge [bend left=15] node {} (S2);
    \path[->] (S4) edge [bend right=15] node {} (S1);
    \path[->] (S4) edge [bend left=15] node {} (S2);
    
    \path[->] (S3) edge [bend left=15,dashed] node {} (S5);
    \path[->] (S3) edge [bend right=15,dashed] node {} (S6);
    \path[->] (S4) edge [bend left=15,dashed] node {} (S5);
    \path[->] (S4) edge [bend left=15,dashed] node {} (S6);
    \path[->] (S3) edge [bend right=15] node {} (S5);
    \path[->] (S4) edge [bend right=15] node {} (S6);
    
    \path[->] (S5) edge [bend left=15,dashed] node {} (S6);
    \path[->] (S6) edge [bend left=15,dashed] node {} (S5);
    \path[->] (S5) edge [loop left,dashed] node {} ();
    \path[->] (S6) edge [loop right,dashed] node {} ();


\end{tikzpicture}
}

\caption{State transition diagram for one arm in the Digital Diabetes domain. Bold (dotted) arrows are transitions when $a=1$ ($a=0$). 
}
\label{fig:transition_diagram}
\end{figure}

%% file: img/transition_diagram_clinical.tex
\begin{figure}
    \centering
\scalebox{0.8}{
\begin{tikzpicture}[node distance=2cm and 1cm,>=stealth',auto, every place/.style={draw,minimum size=2cm,text width=2cm,align=center,node distance=1cm}]
    \node [place] (S1) {$s_E$\\\mbox{$s_C=\bm{C}$}\\$s_M=\bm{M}$};
    \node [place] (S2) [right=of S1] {$s_E$\\\mbox{$s_C=\bm{C}$}\\$s_M=\bm{M}$};
    
    \path[->] (S1) edge [bend left=30,dashdotted] node {$\bm{M_1}=\texttt{Eng.}$} (S2);
    \path[->] (S1) edge [bend right=30,dotted] node[below] {$\bm{M_1}\ne\texttt{Eng.}$} (S2);

\end{tikzpicture}
}
\caption{Construction of the delayed intervention effect on clinical state, $s_C$, via zoomed in view of Fig.~\ref{fig:transition_diagram}. Each transition edge in Fig.~\ref{fig:transition_diagram} encodes \emph{two} transition edges with different probabilities, each of which depend on the engagement state of a patient 3 months ago, i.e., entry 1 of the memory state $\bm{M}$.}
\label{fig:transition_diagram_clinical}
\end{figure}

%% file: 3_proofs.tex
\subsection{Proof of Thm. 2}
\textbf{Thm. 2.}
$L_g^0(\cdot,b)$ is monotone increasing and concave in $b$.

\begin{proof}
We drop the subscript $g$ for ease of notation. Let $\boldsymbol{\lambda}$ be the $H$-length vector of Lagrange multipliers $\lambda^k$ for each timestep. Taking the derivative of $L^0$ with respect to $b$, we have

\begin{equation*}\tag{1}
\frac{dL^0}{db} = H\sum^{H-1}_{k=0}\lambda^k
\end{equation*}
Further, all $\lambda_k\ge 0$ following the Lagrangian relaxation of upper bound constraints in the \textit{max} problem. So
\begin{equation*}
\frac{dL^0}{db} \ge 0
\end{equation*}
Thus $L^0(\cdot,b)$ is monotone increasing in $b$. 

Moreover, the derivative of $L^0$ with respect to $\boldsymbol{\lambda}$ for each timestep is:

\begin{align*}
&\frac{dL}{d\boldsymbol{\lambda}} = \tag{2}\\
&\left[\sum_{n=1}^{N}\mathbb{E}[-a^0_{n}] + b, \sum_{n=1}^{N}\mathbb{E}[-a^1_{n}] + b,\ldots,\sum_{n=1}^{N}\mathbb{E}[-a^{H-1}_{n}] + b\right].
\end{align*}

Note that $\sum_{n=1}^{N}\mathbb{E}[-a^k_n]$ is sum of the expected value of the cost of the actions over all $N$ arms at the timestep $k$, with respect to the optimal policy (derivative of the optimal value function) and the state transition probabilities. Setting Eq.~2 to 0 we get that:
\begin{equation*}\tag{3}
\sum_{n=1}^{N}\mathbb{E}[a^k_n] = b, \hspace{2mm} \forall k \in [1,\ldots,T].
\end{equation*}
In other words, the optimal solution sets the Lagrange multipliers (effectively, "action charge multipliers") $\boldsymbol{\lambda}$ such that the optimal policy takes actions with sum total cost of $b$ in expectation at each timestep. This implies that as $b$ increases, the expected cost of actions taken by the optimal policy each round also increases. 

We now establish further properties of Eq.~1, by establishing how $\boldsymbol{\lambda}$ vary with $b$ at the optimal solution of $L^0$. To do so, we reason about whether $\boldsymbol{\lambda}$ must increase or decrease such that $\sum_{n=1}^{N}\mathbb{E}[a^k_n]$ would increase. 

We can observe from Eq.~4 in the main text that $\boldsymbol{\lambda}$ controls the cost of acting, and that the optimal relaxed arm value functions $V^k_n(\cdot, \boldsymbol{\lambda})$ are piece-wise linear convex functions of $-\boldsymbol{\lambda}$ (max over piece-wise linear convex function is also piece-wise linear convex). This implies that, generally, as $\boldsymbol{\lambda}$ decreases, the policy will take actions with larger values of $a_{nj}^k$ on more arms. However, for this to be true for all values of $\boldsymbol{\lambda}$ is equivalent to Whittle's condition of indexability [Whittle, 1989], a common condition for restless bandit problems. Under this condition, decreasing $\boldsymbol{\lambda}$ implies increasing $\sum_{n=1}^{N}\mathbb{E}[a^k_n]$ and vice versa.

Hence, as $b$ increases, the values of $\lambda^k$ decrease at the optimal solution for $L$. And thus, Eq.~1 is a decreasing function in $b$, implying that $L^0$ is concave in $b$.
\end{proof}

\subsection{Proof of Thm. 3}
\textbf{Thm. 3.}
$L_g^0(\cdot,b) - V_g^0(\cdot,b) < \epsilon$ where $\epsilon = (N-b)H$.

\begin{proof}
Again we drop the subscript $g$ for ease of exposition. For this analysis, we assume arms have the same transition functions and start state $s_n^0$. 

Let the policy which never acts on a given arm be $\underline{\pi}\_n$. Let the policy which always acts on a given arm be $\overline{\pi}\_n$.

First, we establish that the Lagrangian bound is tight at $b=0$.

Observe Eq. 3 from the proof of Thm. 2. At $b=0$:

\begin{equation*}
\sum_{n=1}^{N}\mathbb{E}[a^k_n] = 0, \hspace{2mm} \forall k \in [1,\ldots,T].
\end{equation*}

Since $a^k_{nj} \ge 0$ $\forall n, j, k$, the optimal arm value functions in Eq. 4 of the main text $V^k_n(\cdot, \boldsymbol{\lambda})$ must correspond to policies that always take actions with cost exactly $a^k_{nj} = 0$. This implies that the Lagrange bound gives the policy $\underline{\pi}\_n$ for all arms. This also implies that the values $\lambda^k$ do not affect the solution, and thus the problem is decoupled, since only $\lambda^k$ couples the individual arm value functions $V^k_n(\cdot, \lambda)$. Moreover, since $b=0$, the solution to Eq. 4, reduces to:

\begin{align*}\tag{4}
L^0(\boldsymbol{s}^0,b=0) &= \sum_{n=1}^{N}V_{n}^{0}(s_n^0, \underline{\pi}\_n) + b\sum_{t=0}^H \lambda^t \\
&= \sum_{n=1}^{N}V_{n}^{0}(s_n^0, \underline{\pi}\_n)
\end{align*}

where $V_{n}^{0}(s_n^0, \underline{\pi}\_n)$ is the value function of the policy $\underline{\pi}\_n$ for arm $n$ starting at time $0$ and state $s_n^0$.

Conversely, note that any policy at a given budget level $b$ is a lower bound on the un-relaxed value function $V^0(\boldsymbol{s}^0,b)$. Let $\underline{\pi}$ be the trivial policy which never acts on any arms. Clearly the value function of $\underline{\pi}$ is also given by:

\begin{equation*}\tag{5}
\underline{V}^0(\boldsymbol{s}^0, \underline{\pi}) = \underline{V}^0(\boldsymbol{s}^0, b=0)  = \sum_{n=1}^{N}V_{n}^{0}(s_n^0, \underline{\pi}\_n).
\end{equation*}

Since $L^0$ upper bounds $V^0$, $\underline{V}^0$ lower bounds $V^0$, and since $L^0(\cdot,b=0) = \underline{V}^0(\cdot, b=0)$, the upper bound $L^0$ is tight at $b=0$.

Next, we establish that the Lagrangian bound is tight at $b=N$. Assume, without loss of generality, that $a^k_{nj} \le 1$. Then Eq. 3 from the proof of Thm. 2 at $b=N$ gives:

\begin{equation*}
\sum_{n=1}^{N}\mathbb{E}[a^k_n] = N, \hspace{2mm} \forall k \in [1,\ldots,T].
\end{equation*}

In words, each $\lambda^k$ should be set such that the optimal policy takes the most expensive action on every arm in expectation. This is achieved when the action charge multipliers $\lambda^k = 0$ $\forall k$. Alternatively, this can be interpreted as the shadow price of further relaxation of a constraint which is not tight. I.e., when $b=N$ and $a^k_{nj} \le 1$, there is effectively no budget constraint. When all $\lambda^k$ are 0, then the relaxed problem is again decoupled, since $\lambda^k$ are the only terms that couple arms in the Lagrangian relaxation.

So for $b=N$, the solution to Eq. 4 reduces to:

\begin{align*}\tag{6}
L^0(\boldsymbol{s}^0,b=N) &= \sum_{n=1}^{N}V_{n}^{0}(s_n^0, \overline{\pi}\_n) + b\sum_{t=0}^T \lambda^t \\
&= \sum_{n=1}^{N}V_{n}^{0}(s_n^0, \overline{\pi}\_n)
\end{align*}

Conversely, let $\overline{\pi}$ be the policy which always takes the most expensive action on all arms. Clearly the value function of $\overline{\pi}$ is also given by:

\begin{align*}\tag{7}
\underline{V}^0(s_n^0,\overline{\pi}) = \underline{V}^0(s_n^0,b=N) = \sum_{n=1}^{N}V_{n}^{0}(s_n^0, \overline{\pi}\_n)
\end{align*}

Similarly, $L^0(\cdot,b=N) = \underline{V}^0(\cdot, b=N)$, and so the upper bound $L^0$ is tight at $b=N$. Moreover, since $L^0(\cdot,b)$ is monotone increasing (Thm. 2):

\begin{align*}\tag{9}
L^0(\boldsymbol{s}^0,b) \le \sum_{n=1}^{N} V_{n}^{0}(s_n^0, \overline{\pi}\_n) \hspace{2mm} \text{ where } 0 \le b \le N
\end{align*}

Now we consider a lower bound on $V^0(\boldsymbol{s},b)$, namely $\underline{V}^0(\boldsymbol{s}^0,b)$. It is convenient to analyze the value function of a policy which acts on the same $b$ arms each round, consisting of per-arm value functions $V_n(s^0_n, \overline{\pi}\_n)$ which act every round and $V_n(s^0_n, \underline{\pi}\_n)$ which never act:

\begin{align*}\tag{10}
\underline{V}^0(\boldsymbol{s}^0,b) = \sum_{n=1}^b V_{n}^{0}(s_n^0, \overline{\pi}\_n) + \sum_{n=b}^N V_{n}^{0}(s_n^0, \underline{\pi}\_n)
\end{align*}

Then to bound the gap between $L^0$ and $V^0$, it is sufficient to bound the gap between $L^0$ and $\underline{V}^0$. That is, $L^0(\cdot,b) - V^0(\cdot,b) = \epsilon$, where:

\begin{align*}\tag{11}
\epsilon &\le \sum_{n=1}^N V_{n}^{0}(s_n^0, \overline{\pi}\_n) - \sum_{n=1}^{b}V_{n}^{0}(s_n^0, \overline{\pi}\_n) - \sum_{n=b}^N V_{n}^{0}(s_n^0, \underline{\pi}\_n) \\
 &\le \sum_{n=b}^N V_{n}^{0}(s_n^0, \overline{\pi}\_n) - V_{n}^{0}(s_n^0, \underline{\pi}\_n) \\
 &\le (N-b)H
\end{align*}

Where the final step comes from assuming, without loss of generality, that all per-step rewards are between 0 and 1.
\end{proof}

\subsection{Thm. 4 Proof}

\textbf{Thm. 4}
The following inequalities hold $\hspace{1mm} \forall g, \hspace{1mm} \forall a \in M^{-1}(g), \hspace{1mm} \forall b$:

\begin{equation*}\tag{12}
v^a(b) \le V_g(b) \le |M^{-1}(g)| \max_{c\in M^{-1}(g)}v^c(b).
\end{equation*}

\begin{proof}
Recall that we define 
\textit{arm-value functions}, namely, $v^a(b)$, which capture the value function of the single arm $a$, given budget $b$. These are in contrast to \textit{group-value functions} $V_g(b)$ which are the value function of the optimal RMAB policy over all arms $a\in|M^{-1}(g)|$ given per-round budget $b$. We define a composite function $h_g$ to characterize the relationship between the arm-value functions in a group $g$ and the group-value function of the group $g$.

Specifically, $h_g$ maps the set of all arm-value functions in a group $\mathcal{A} = \{v^a : a \in M^{-1}(g)\}$ to group-value functions $V_g$. E.g., if group $g$ has two arms $c$ and $d$, then $V_g(b) = h_g(\{v^c, v^d\})(b)$. We are interested in the relationship of $h_g$ and one of its input arm value functions $v^a$. 

First, we consider the lower bound on $h_g$.  Consider all sets $\mathcal{A}$ such that $|\mathcal{A}| > 2$. Similar, to the technique used in the proof of Thm. 3, we can consider the value of a specific sub-optimal policy to measure the lower bound. Specifically, the optimal value of $h_g(\mathcal{A})(b) = V_g(b)$ is lower bounded by a policy which gives $b$ budget to just one of its arms and 0 to all others. Formally, we have:
\begin{align*}\tag{13}
h_g(\mathcal{A})(b) &\ge v^a(b) + \sum_{i=1}^{|\mathcal{A}|-1}v^{(\mathcal{A} - \{a\})(i)}(0) \hspace{2mm} \forall a \in \mathcal{A} \\
&\ge v^a(b) \hspace{2mm} \forall a \in \mathcal{A}
\end{align*}

Then, we consider the upper bound. The results on the monotone increasing structure of $V_g(b)$ hold also for arm-value functions $v^a(b)$ (consider $V_g(b)$ with group size of 1). Since the composite function $h_g(\mathcal{A})(b)$ causes a limited budget of $b$ to be shared across many arms, the composite function $h_g$ can be seen as reducing the \textit{effective} budget of each individual $v^a$ included in the composite. Alternatively, we can show that $h_g$ gives functions that are upper bounded by composites of arm value functions that \textit{do not} have to share budget. Formally:
\begin{align*}\tag{14}
h_g(\mathcal{A})(b) &\le \sum_{a \in \mathcal{A}} v^a(1) \hspace{3mm} \text{always act on all arms} \\
&\le \sum_{a \in \mathcal{A}} v^a(b) \hspace{2mm} \forall b \ge 1 \\
&\le |\mathcal{A}|\max_{a}v^a(b).
\end{align*}
This gives the proof and the key result that $h_g$ is a composite that scales more slowly in $b$ than the sum of its input functions. 
\end{proof}